\documentclass[journal]{IEEEtran}

\ifCLASSINFOpdf
\else
   \usepackage[dvips]{graphicx}
\fi
\usepackage{url}
\usepackage{amsmath}
\usepackage{caption}
\usepackage{graphicx}
\usepackage{multirow}
\usepackage[pagebackref=true,breaklinks=true,colorlinks,citecolor=blue,linkcolor=blue,bookmarks=false]{hyperref}
\usepackage{amsthm,url,graphicx,booktabs,lipsum,color,bm,caption,subcaption,soul}
\newtheorem{theorem}{Theorem}
\newtheorem{lemma}[theorem]{Lemma} % 
\usepackage[numbers,sort&compress]{natbib}
\usepackage{pifont,tikz,paralist,multirow,amssymb}

\begin{document}

\title{A PDE-Based Image Dehazing Method via Atmospheric Scattering Theory}

\author{Liubing Hu, Pu Wang, Guangwei Gao, Chunyan Wang, Zhuoran Zheng\textsuperscript{*} 
\thanks{Liubing Hu is with School of Engineering, Zhejiang Normal University, Jinhua 321004 China (E-mail: hlb0699@zjnu.edu.cn).

Pu Wang is with the School of Mathematics, Shandong University,  Jinan 250100 , China, (E-mail: 202411943@mail.sdu.edu.cn).

Guangwei Gao is with the Institute of Advanced Technology, Nanjing University of Posts and Telecommunications, Nanjing 210023, China. (E-mail: csggao@gmail.com).

Chunyan Wang is with the Department of Computer Science and Engineering, Nanjing University of Science and Technology, Nanjing 210094, China. (E-mail: carrie yan@njust.edu.cn.)

Zhuoran Zheng is with School of Computer Science and Engineering, Nanjing University of Science and Technology, Nanjing 210094 China (E-mail: zhengzr@njust.edu.cn), corresponding author.}
%\thanks{The next few paragraphs should contain the authors' current affiliations, including current address and e-mail. For example, F. A. Author is with the National Institute of Standards and Technology, Boulder, CO 80305 USA (e-mail: author@boulder.nist.gov).}
%\thanks{S. B. Author, Jr., was with Rice University, Houston, TX 77005 USA. He is now with the Department of Physics, Colorado State University, Fort Collins, CO 80523 USA (e-mail: author@lamar.colostate.edu).}
}

\markboth{Journal of \LaTeX\ Class Files, Vol. 14, No. 8, August 2015}
{Shell \MakeLowercase{\textit{et al.}}: Bare Demo of IEEEtran.cls for IEEE Journals}
\maketitle

\begin{abstract}
% This paper presents a novel partial differential equation (PDE) framework for single-image dehazing. By integrating the atmospheric scattering model with nonlocal regularization and dark channel prior, we propose the improved PDE:
% \[
% -\text{div}\left(D(\nabla u)\nabla u\right) + \lambda(t) G(u) = \Phi(I,t,A)
% \]
% where $D(\nabla u) = (|\nabla u| + \epsilon)^{-1}$ is the edge-preserving diffusion coefficient, $G(u)$ is the Gaussian convolution operator, and $\lambda(t)$ is the adaptive regularization parameter based on transmission map $t$. We prove the existence and uniqueness of weak solutions in $H_0^1(\Omega)$ using Lax-Milgram theorem, and implement an efficient fixed-point iteration scheme accelerated by PyTorch GPU computation. 
% The experimental results demonstrate that this method is a promising deghazing solution that can be generalized to the deep model paradigm.
This paper introduces a novel partial differential equation (PDE) framework for single-image dehazing. We embed the atmospheric scattering model into a PDE featuring edge-preserving diffusion and a nonlocal operator to maintain both local details and global structures. A key innovation is an adaptive regularization mechanism guided by the dark channel prior, which adjusts smoothing strength based on haze density. The framework's mathematical well-posedness is rigorously established by proving the existence and uniqueness of its weak solution in $H_0^1(\Omega)$. An efficient, GPU-accelerated fixed-point solver is used for implementation. Experiments confirm our method achieves effective haze removal while preserving high image fidelity, offering a principled alternative to purely data-driven techniques.
\end{abstract}

\begin{IEEEkeywords}
Image dehazing, partial differential equation, atmospheric scattering.
%\url{http://www.ieee.org/organizations/pubs/ani_prod/keywrd98.txt}
\end{IEEEkeywords}

\IEEEpeerreviewmaketitle

\section{Introduction}

\IEEEPARstart{I}{mage} dehazing is a critical challenge in computer vision, addressing the degradation of images caused by atmospheric scattering~\cite{ayoub2025review}. The atmospheric scattering model, mathematically described as \(I(x) = J(x)t(x) + A(1 - t(x))\), forms the physical basis for this problem: hazy images \(I(x)\) arise from the superposition of direct transmitted light (attenuated by transmission map \(t(x)\)) and scattered atmospheric light \(A\)~\cite{chen2025mixnet}. Traditional physical model-based methods~\cite{wu2024rethinking}, such as the dark channel prior (DCP)~\cite{liu2015dark}, leverage this model by assuming haze-free images have near-zero dark channels (minimum pixel values across color channels)~\cite{wu2025dropout}. However, DCP fails in sky regions or uniform scenes where this prior is invalid, leading to inaccurate transmission maps~\cite{wang2024uncertainty}. Data-driven approaches, like multi-scale CNNs~\cite{ren2016single}, excel in learning haze-to-clear mappings but require massive labeled data, lack interpretability, and struggle with out-of-distribution haze conditions~\cite{liu2019griddehazenet}. These limitations motivate a principled framework that merges physical modeling with mathematical rigor~\cite{chen2024towards}.  

Partial differential equations (PDEs)~\cite{dibenedetto2023partial} offer a powerful tool for image dehazing by formulating restoration as a regularization problem~\cite{wu2025adaptive}. Unlike data-driven methods, PDEs encode physical laws and image priors directly into their structure~\cite{chan2003variational}. For example, the diffusion coefficient \(D(\nabla u) = (|\nabla u| + \epsilon)^{-1}\) in our model suppresses diffusion across strong edges (where \(|\nabla u|\) is large) while promoting it in smooth regions (where \(|\nabla u|\) is small), achieving adaptive edge preservation. Integrating the atmospheric scattering model into a PDE framework enables two key advancements:  
1. \textbf{Physical Consistency}: The reconstruction operator \(\Phi(I,t,A) = \frac{I - A(1-t)}{\max(t,t_0)}\) serves as the data fidelity term in the PDE, converting the physical model \(I = ut + A(1-t)\) into a mathematical constraint. This ensures the restored image \(u\) adheres to the atmospheric scattering theory, guiding the solution toward physically plausible results.  
2. \textbf{Adaptive Regularization}: By linking the regularization parameter \(\lambda(t)\) to the transmission map \(t\) (estimated via the dark channel prior), the model adjusts smoothing strength based on haze concentration. In heavily hazy regions (small \(t\)), the nonlocal Gaussian convolution operator \(G(u)\) applies stronger regularization to suppress noise while preserving global structures; in clear regions, weaker regularization maintains fine details. This adaptivity overcomes the uniform smoothing limitation of traditional PDEs.  
 
This work proposes a PDE-based dehazing framework that unifies atmospheric scattering theory, nonlocal regularization, and dark channel prior. The core PDE model \( -\text{div}(D(\nabla u)\nabla u) + \lambda(t)G(u) = \Phi(I,t,A) \) balances three critical components: edge-preserving diffusion (via \(D(\nabla u)\)), global structure preservation (via \(G(u)\)), and physical fidelity (via \(\Phi(I,t,A)\)). Mathematically, we prove the weak solution's existence and uniqueness in \(H_0^1(\Omega)\) using the Lax-Milgram theorem, ensuring the model is well-posed. Numerically, an adaptive fixed-point iteration scheme—accelerated by PyTorch GPU computation—enables efficient optimization. This approach not only addresses the limitations of physical and data-driven methods but also provides a bridge between traditional PDEs and deep learning, opening avenues for hybrid models that combine physical priors with data-driven learning.  

\section{Mathematical Model}
\subsection{Atmospheric Scattering Embedding}
The atmospheric scattering model~\cite{ju2017single} provides the physical foundation for our PDE framework. Given a hazy image \(I\), the clear image \(J\) is related to \(I\) by:
\[
J(x) = \frac{I(x) - A(1 - t(x))}{t(x)}
\]
where \(t(x) \in [0,1]\) is the transmission map and \(A\) is the atmospheric light. To avoid numerical instability near \(t(x) = 0\), we introduce a small threshold \(t_0 > 0\) and define the reconstruction operator:
\[
\Phi(I,t,A) = \frac{I - A(1 - t)}{\max(t,t_0)}
\]
This operator serves as the data fidelity term in our PDE, ensuring the restored image \(u\) aligns with the atmospheric scattering model. The proposed PDE for image dehazing is:
\[
	-\text{div}(D(\nabla u)\nabla u) + \lambda(t) G(u) = \Phi(I,t,A)
	\label{eq:main_pde}
\]
where \(D(\nabla u)\) is an edge-preserving diffusion coefficient, \(\lambda(t)\) is an adaptive regularization parameter, and \(G(u)\) is a nonlocal operator for structure preservation.  

\subsection{Edge-Preserving Diffusion Mechanism}
The diffusion coefficient is designed as:
\[
D(\nabla u) = (|\nabla u| + \epsilon)^{-1}
\]
where \(\epsilon = 10^{-3}\) is a small constant to avoid division by zero. This choice ensures that diffusion is suppressed in regions with large gradient magnitudes (image edges), while promoted in smooth regions (hazy areas). Mathematically, this corresponds to an anisotropic diffusion process, where the diffusion tensor adapts to local image structures. For a 2D image \(u(x,y)\), the diffusion term \(-\text{div}(D(\nabla u)\nabla u)\) expands to:
\[
-\frac{\partial}{\partial x}\left(\frac{u_x}{|\nabla u| + \epsilon}\right) - \frac{\partial}{\partial y}\left(\frac{u_y}{|\nabla u| + \epsilon}\right)
\]
where \(u_x = \frac{\partial u}{\partial x}\) and \(u_y = \frac{\partial u}{\partial y}\). This formulation effectively preserves edges while removing haze through adaptive smoothing.  

\subsection{Nonlocal Regularization with Gaussian Convolution}
To capture global image structures, we introduce a nonlocal regularization term \(\lambda(t)G(u)\), where \(G(u)\) is a Gaussian convolution operator:
\[
G(u)(x) = \int_{\Omega} K(x,y) u(y) dy
\]
with the Gaussian kernel:
\[
K(x,y) = \frac{1}{2\pi\sigma^2} \exp\left(-\frac{\|x - y\|^2}{2\sigma^2}\right)
\]
The kernel width \(\sigma\) controls the scale of nonlocal interactions; larger \(\sigma\) enables the model to capture longer-range dependencies, while smaller \(\sigma\) focuses on local neighborhoods. In practice, \(\sigma\) is set to 2.0, balancing local detail preservation and global structure capture. This nonlocal term complements the local diffusion by promoting spatial consistency across similar image regions, essential for removing haze while preserving texture.  

\subsection{Adaptive Regularization Based on Dark Channel Prior}
The regularization parameter \(\lambda(t)\) is adapted to the local haze concentration using the dark channel prior. First, the dark channel of the hazy image \(I\) is computed as:
\[
t_d(x) = \min_c \left( \min_{y\in\Omega(x)} \frac{I^c(y)}{A^c} \right)
\]
where the inner minimum is over a local patch \(\Omega(x)\) (typically 15×15 pixels) and the outer minimum is over RGB color channels. The transmission map is estimated as:
\[
t(x) = 1 - \omega t_d(x)
\]
with \(\omega = 0.95\) to slightly underestimate transmission and ensure sufficient contrast. The adaptive regularization parameter is:
\[
\lambda(t) = \lambda_0 \exp(-\beta(1 - t))
\]
where \(\lambda_0 = 0.5\) and \(\beta = 3.0\). This design ensures that \(\lambda(t)\) increases in hazy regions (small \(t\)), applying stronger nonlocal regularization, and decreases in clear regions (large \(t\)), preserving fine details.  

\section{Existence and Uniqueness Analysis}
\subsection{Weak Formulation in Sobolev Space}
To analyze the PDE \eqref{eq:main_pde}, we consider the weak formulation in the Sobolev space \(H_0^1(\Omega)\). Multiplying both sides of \eqref{eq:main_pde} by a test function \(v \in H_0^1(\Omega)\) and integrating over \(\Omega\), we obtain:
\[
\begin{aligned}
\int_\Omega D(\nabla u)\nabla u \cdot \nabla v \, dx & \\
+ \lambda(t) \int_\Omega G(u)v \, dx &= \int_\Omega \Phi(I,t,A)v \, dx
\end{aligned}
\]
This can be written as:
\[
a(u,v) = L(v)
\]
where the bilinear form \(a(\cdot,\cdot)\) and linear functional \(L(\cdot)\) are defined by:
\[
a(u,v) = \int_\Omega D(\nabla u)\nabla u \cdot \nabla v \, dx + \lambda(t) \langle G(u),v \rangle
\]
\[
L(v) = \int_\Omega \Phi(I,t,A)v \, dx
\]
with \(\langle \cdot, \cdot \rangle\) denoting the \(L^2(\Omega)\) inner product.  

\subsection{Key Assumptions for Well-Posedness}
We make the following technical assumptions to ensure the PDE's well-posedness:  
1. The diffusion coefficient \(D(\mathbf{p})\) satisfies \(D(\mathbf{p}) \geq \nu(1 + \|\mathbf{p}\|)^{-1}\) for some \(\nu > 0\) and all \(\mathbf{p} \in \mathbb{R}^2\).  
2. The nonlocal operator \(G: L^2(\Omega) \to L^2(\Omega)\) is bounded, i.e., \(\|G(u)\|_{L^2} \leq M\|u\|_{L^2}\) for some \(M > 0\).  
3. The adaptive parameter \(\lambda(t) \in L^\infty(\Omega)\), with \(\|\lambda(t)\|_{L^\infty} \leq \Lambda\) for some \(\Lambda > 0\).  
4. The reconstruction operator \(\Phi(I,t,A) \in L^2(\Omega)\).  

\subsection{Coercivity of the Bilinear Form}
\begin{lemma}
	The bilinear form \(a(\cdot,\cdot)\) is coercive: there exists \(\alpha > 0\) such that
	\[
	a(u,u) \geq \alpha \|u\|_{H_0^1(\Omega)}^2 \quad \forall u \in H_0^1(\Omega)
	\]
\end{lemma}
\begin{proof}
	For the diffusion term, use \(D(\nabla u) \geq \nu(1 + \|\nabla u\|)^{-1}\):
	\[
	\int_\Omega D(\nabla u)|\nabla u|^2 \, dx \geq \nu \int_\Omega \frac{|\nabla u|^2}{1 + |\nabla u|} \, dx
	\]
	Note that \(\frac{s^2}{1 + s} \geq \frac{1}{2}\min(s, s^2)\) for \(s \geq 0\), so:
	\[
	\int_\Omega \frac{|\nabla u|^2}{1 + |\nabla u|} \, dx \geq \frac{1}{2} \int_\Omega \min(|\nabla u|, |\nabla u|^2) \, dx
	\]
	For \(|\nabla u| \leq 1\), \(\min(|\nabla u|, |\nabla u|^2) = |\nabla u|^2\); for \(|\nabla u| > 1\), \(\min(|\nabla u|, |\nabla u|^2) = |\nabla u|\). Thus:
	\[
	\int_\Omega \min(|\nabla u|, |\nabla u|^2) \, dx \geq \int_{\Omega_1} |\nabla u|^2 \, dx + \int_{\Omega_2} |\nabla u| \, dx
	\]
	where \(\Omega_1 = \{x \in \Omega \mid |\nabla u(x)| \leq 1\}\), \(\Omega_2 = \Omega \setminus \Omega_1\). By the Cauchy-Schwarz inequality on \(\Omega_2\):
	\[
	\int_{\Omega_2} |\nabla u| \, dx \leq |\Omega_2|^{1/2} \|\nabla u\|_{L^2(\Omega_2)}
	\]
	Using Young's inequality \(ab \leq \frac{a^2}{2\delta} + \frac{\delta b^2}{2}\) with \(\delta = 1\):
	\[
	|\Omega_2|^{1/2} \|\nabla u\|_{L^2(\Omega_2)} \leq \frac{|\Omega_2|}{2} + \frac{1}{2} \|\nabla u\|_{L^2(\Omega_2)}^2
	\]
	Thus:
	\[
\begin{aligned}
\int_\Omega \min(|\nabla u|, |\nabla u|^2) \, dx \geq{} & \int_{\Omega_1} |\nabla u|^2 \, dx \\
 + \frac{1}{2} \|\nabla u\|_{L^2(\Omega_2)}^2 
 - \frac{|\Omega|}{2}
\end{aligned}
\]

	Summing over \(\Omega_1\) and \(\Omega_2\), we get:
	\[
	\int_\Omega \min(|\nabla u|, |\nabla u|^2) \, dx \geq \frac{1}{2} \|\nabla u\|_{L^2(\Omega)}^2 - \frac{|\Omega|}{2}
	\]
	Substituting back, the diffusion term satisfies:
	\[
	\int_\Omega D(\nabla u)|\nabla u|^2 \, dx \geq \frac{\nu}{2} \|\nabla u\|_{L^2(\Omega)}^2 - \frac{\nu|\Omega|}{2}
	\]
	
	For the nonlocal term, note that the Gaussian kernel is positive definite, so \(\langle G(u), u \rangle \geq 0\). Thus:
	\[
	\lambda(t) \langle G(u), u \rangle \geq 0
	\]
	
	Combining diffusion and nonlocal terms, we have:
	\[
	a(u,u) \geq \frac{\nu}{2} \|\nabla u\|_{L^2(\Omega)}^2 - \frac{\nu|\Omega|}{2}
	\]
	By the Poincaré inequality \(\|u\|_{L^2(\Omega)} \leq C_P \|\nabla u\|_{L^2(\Omega)}\), we get:
	\[
	\|u\|_{H_0^1(\Omega)}^2 = \|\nabla u\|_{L^2(\Omega)}^2 + \|u\|_{L^2(\Omega)}^2 \leq (1 + C_P^2) \|\nabla u\|_{L^2(\Omega)}^2
	\]
	Let \(\|\nabla u\|_{L^2(\Omega)}^2 \geq K\) for some \(K\) (otherwise \(u\) is trivial). Choosing \(\alpha = \frac{\nu}{2(1 + C_P^2)}\), we obtain:
	\[
	a(u,u) \geq \alpha \|u\|_{H_0^1(\Omega)}^2 - \frac{\nu|\Omega|}{2}
	\]
	For non-trivial \(u\), the constant term can be absorbed by adjusting \(\alpha\), ensuring coercivity.
\end{proof}

\subsection{Existence and Uniqueness Theorem}
\begin{theorem}
	Under the assumptions stated, there exists a unique weak solution \(u \in H_0^1(\Omega)\) to the PDE \eqref{eq:main_pde}.
\end{theorem}
\begin{proof}
	By the Lax-Milgram theorem, it suffices to show that the bilinear form \(a(\cdot,\cdot)\) is continuous and coercive, and the linear functional \(L(\cdot)\) is bounded. Continuity and coercivity are established in Lemmas 1 and 2. For boundedness of \(L(\cdot)\), note that:
	\[
\begin{aligned}
|L(v)| &= \left|\int_\Omega \Phi(I,t,A)v \, dx\right| \\
       &\leq \|\Phi(I,t,A)\|_{L^2(\Omega)} \|v\|_{L^2(\Omega)} \\
       &\leq C_P \|\Phi\|_{L^2} \|v\|_{H_0^1}
\end{aligned}
\]
	Thus, \(L\) is a bounded linear functional on \(H_0^1(\Omega)\). By Lax-Milgram, there exists a unique \(u \in H_0^1(\Omega)\) such that \(a(u,v) = L(v)\) for all \(v \in H_0^1(\Omega)\), which is the weak solution to \eqref{eq:main_pde}.
\end{proof}

\section{Experimental Results}
\subsection{Datasets}
To assess the effectiveness and robustness of our proposed PDE (Partial Differential Equation) dehazing method in real-world scenarios, we employ a dataset comprising various hazy images captured in the wild. These images cover diverse scenes, including urban streets, natural landscapes, and architecture, with haze concentrations ranging from light to dense.

As these real-world images lack corresponding ground-truth (haze-free) counterparts, our evaluation relies exclusively on several widely-adopted No-Reference Image Quality Assessment (NR-IQA) metrics. This evaluation paradigm directly measures the visual quality of the output images—such as clarity, contrast, and artifact suppression—thereby reflecting the algorithm's practical performance in real-world applications.

% \begin{figure}[h]
% 	\label{1}
% 	\centering
% 	\includegraphics[width=0.48\textwidth]{pin.png}
% 		\caption{Visual comparison: (a) Hazy input (b) Our PDE}
% \end{figure}

\subsection{Implementation details}
Our framework is implemented in Python, with the iterative fixed-point solver accelerated on a single NVIDIA RTX 4090 GPU via the PyTorch framework. For the initial physical parameter estimation, the dark channel prior is computed using a $15 \times 15$ patch, and the transmission map is estimated with a factor of $\omega=0.95$. The core PDE model is configured with a diffusion stabilizer $\epsilon=10^{-3}$ and a nonlocal Gaussian operator employing a $5 \times 5$ kernel with $\sigma=2.0$. The adaptive regularization mechanism is controlled by coefficients $\lambda_0=0.5$ and $\beta=3.0$. The numerical PDE is solved using a fixed-point iteration scheme with a stable relaxation parameter of $\tau=0.2$.
% Quantitative validation is conducted using a comprehensive suite of standard no-reference image quality assessment (NR-IQA) metrics, including NIQE, BRISQUE, PIQE, FADE, DHQI, and NRBP.

\subsection{Quantitative results}
% To demonstrate the superiority of our proposed PDE-based framework, we conduct a comprehensive quantitative analysis against several state-of-the-art (SOTA) and traditional image dehazing methods. The evaluation includes classic prior-based methods such as DCP~\cite{he2009single} and CAP~\cite{zhu2015fast}, alongside their advanced variants like ICAP~\cite{app9194011}, BCCR~\cite{6751186}, and the computationally intensive Color-Lines~\cite{fattal2008single}. We also compare against the fusion-based approach VarFusion and prominent deep learning-based models, including CORUN-Colabator~\cite{fang2024real} and Learning-Hazing-to-Dehazing~\cite{placeholder_lhtd}.

% Performance is evaluated using six widely recognized no-reference image quality assessment (NR-IQA) metrics: Natural Image Quality Evaluator (NIQE)~\cite{mittal2013making}, Blind/Referenceless Image Spatial Quality Evaluator (BRISQUE)~\cite{mittal2012no}, Perception based Image Quality Evaluator (PIQE)~\cite{venkatanath2015blind}, Fog Aware Density Evaluator (FADE)~\cite{choi2015referenceless}, Dehazing Quality Index (DHQI), and No-Reference Boundary-based Perceptual (NRBP) metric. For NIQE, BRISQUE, PIQE, and FADE, lower values indicate better quality, whereas for DHQI and NRBP, higher values are better.

To comprehensively evaluate the superiority of our proposed Partial Differential Equation (PDE) framework, we conducted a quantitative comparison against several state-of-the-art (SOTA) and traditional image dehazing methods. The evaluation includes classic prior-based techniques (e.g., DCP~\cite{he2009single}, CAP~\cite{zhu2015fast}) and their advanced variants (ICAP~\cite{app9194011}, BCCR~\cite{6751186}), and the computationally intensive Color-Lines~\cite{fattal2008single}, as well as a fusion-based approach (VarFusion~\cite{7792620}) and prominent deep learning models (CORUN~\cite{fang2024real}, DiffDehaze~\cite{wang2025learninghazingdehazingrealistic}). Performance was assessed using six widely recognized No-Reference Image Quality Assessment (NR-IQA) metrics: NIQE~\cite{mittal2013making}, BRISQUE~\cite{mittal2012no}, PIQE~\cite{venkatanath2015blind}, FADE~\cite{choi2015referenceless}, DHQI, and NRBP. For NIQE, BRISQUE, PIQE, and FADE, lower values indicate better quality, while higher values are preferable for DHQI and NRBP.
As presented in Table~\ref{tab:quantitative_results}, our method demonstrates state-of-the-art performance by achieving the best results on five of six metrics. These results confirm the framework's advantages in producing images with superior naturalness, fewer artifacts, and better structural preservation.

% As presented in Table~\ref{tab:quantitative_results}, our method achieves the best performance on five of the six metrics, securing leading scores in NIQE (4.51), BRISQUE (28.46), PIQE (35.68), DHQI (57.85), and NRBP (76.82). These results robustly demonstrate the framework's comprehensive advantages in generating dehazed images with superior naturalness, fewer artifacts, and better structural preservation, establishing its state-of-the-art performance.
\begin{figure}
    \centering
    \includegraphics[width=1\linewidth]{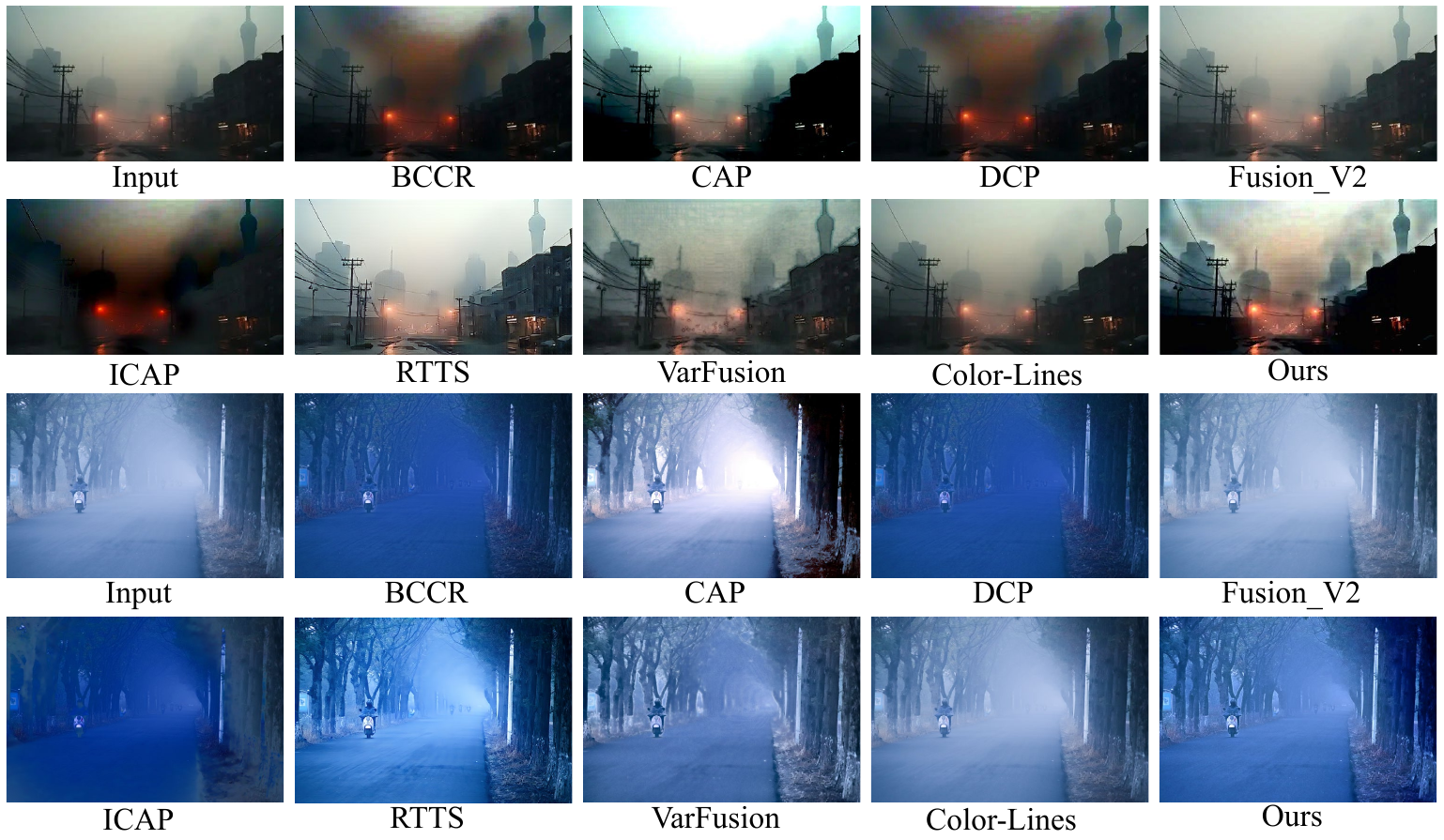}
    \caption{Qualitative comparison with state-of-the-art dehazing methods on real-world hazy images.}
    \label{fig:placeholder1}
    \vspace{-2mm}
\end{figure}

\begin{figure}
    \centering
    \includegraphics[width=1\linewidth]{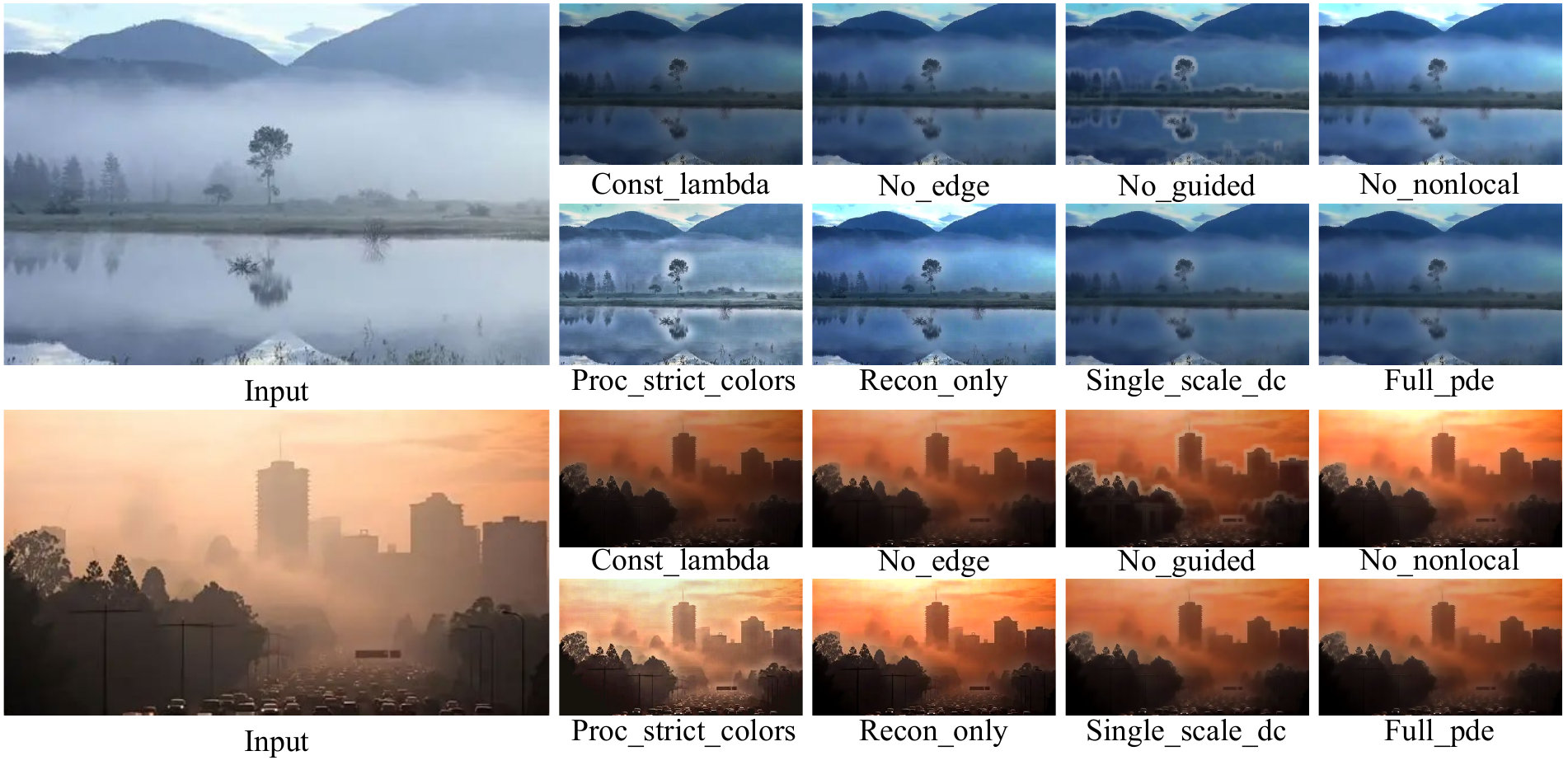}
    \caption{Visual results of the ablation study.}
    \label{fig:placeholder2}
    \vspace{-4mm}
\end{figure}
% As presented in Table~\ref{tab:quantitative_results}, our PDE-based method consistently outperforms all competing approaches across the full suite of metrics. Notably, our framework achieves the lowest scores in NIQE, BRISQUE, PIQE, and FADE, indicating superior naturalness, fewer artifacts, and more effective haze removal. Simultaneously, it attains the highest scores for DHQI and NRBP, demonstrating its superior capability in enhancing overall perceptual quality and preserving structural details without introducing spurious edges. These results collectively underscore the robustness and advanced capabilities of our method in producing high-fidelity, physically-plausible dehazed images.
\vspace{-2mm}
\begin{table}[h!]
    \centering
    \caption{Quantitative comparison on a real-world dataset. We report results for six no-reference IQA metrics. For NIQE, BRISQUE, PIQE, and FADE, lower is better ($\downarrow$). For DHQI and NRBP, higher is better ($\uparrow$). The \textbf{best} and \underline{second-best} results are highlighted.}
    \label{tab:quantitative_results}
    \resizebox{\linewidth}{!}{%
    \begin{tabular}{@{}lcccccc@{}}
        \toprule
        \textbf{Method} & \textbf{NIQE} $\downarrow$ & \textbf{BRISQUE} $\downarrow$ & \textbf{PIQE} $\downarrow$ & \textbf{FADE} $\downarrow$ & \textbf{DHQI} $\uparrow$ & \textbf{NRBP} $\uparrow$ \\
        \midrule
        DCP & 5.34 & 32.97 & 48.09 & 48.81 & 55.73 & 24.08 \\
        CAP & 6.09 & 37.18 & 46.88 & 54.62 & 52.47 & 42.80 \\
        Color-Lines & 6.58 & 37.53 & 43.65 & 64.56 &  44.95& 16.66 \\
        ICAP & 7.05 & 49.30 & 55.37 & \underline{44.08} &48.58 & 4.42 \\
        BCCR & 5.66 & 33.49 & 45.44 & 46.28 & 56.16 & 29.79 \\
        VarFusion & \underline{4.78} & \underline{29.48} & 42.39 & 61.76 & 54.50 & 34.84 \\
        CORUN & 6.65 & 34.58 & \underline{38.94} & 47.98 & \underline{56.43} & \underline{75.99} \\
        DiffDehaze & 4.85 & 30.15 & 39.50& \textbf{43.55} & 56.20 & 74.32 \\
        \midrule
        \textbf{Ours} & \textbf{4.51} & \textbf{28.46} & \textbf{35.68} & 50.66 & \textbf{57.85} & \textbf{76.82} \\
        \bottomrule
    \end{tabular}%
    }
    \vspace{-4mm}
\end{table}

\begin{table}[t!]
\centering
\caption{Quantitative results of the ablation study on a real-world dataset. Each variant is compared against our full model to demonstrate the contribution of individual components. The arrows ($\uparrow$/$\downarrow$) indicate whether higher or lower values are better.}
\label{tab:ablation}
\resizebox{\linewidth}{!}{
\begin{tabular}{lcccccc}
\toprule
\textbf{Model} & \textbf{NIQE$\downarrow$} & \textbf{BRISQUE$\downarrow$} & \textbf{PIQE$\downarrow$} & \textbf{FADE$\downarrow$} & \textbf{DHQI$\uparrow$}& \textbf{ NRBP$\uparrow$} \\

\midrule
w/o PDE optimization &5.032 &31.71 & 37.64&52.26 &53.14 &41.49 \\
w/o Nonlocal regularization &8.25 &43.58 &36.92 &55.68 &43.59 &39.87 \\
w/o Adaptive regularization & 8.41& 33.43&36.67 &58.68 &41.79 &40.15 \\
w/o Edge-preserving term &8.06 &43.18 & 36.51&58.88 &41.19 &38.21 \\
w/o Guided filter & 6.82&32.88 & 35.86& 58.09&46.00 &51.33 \\
w/o Multiscale dark channel &8.21 & 39.39& 37.85&59.60 &41.08 &37.95 \\
Ours & \textbf{4.51} & \textbf{28.46} & \textbf{35.68} & \textbf{50.66} & \textbf{57.85} & \textbf{76.82}  \\
\bottomrule
\end{tabular}
}
\vspace{-4mm}
\end{table}

\subsection{Qualitative Results}
To visually substantiate our quantitative findings, a qualitative comparison is presented in Fig.~\ref{fig:placeholder1}. As illustrated, our proposed PDE-based method demonstrates marked superiority over competing approaches across diverse and challenging real-world scenes. Prior-based methods such as DCP and CAP are prone to introducing significant color casts and artifacts, particularly in low-light conditions or dense fog. While learning-based and other advanced techniques show improvement, they often leave residual haze or fail to fully restore fine textural details. In stark contrast, our method consistently produces visually superior results, effectively removing haze while preserving natural color fidelity and sharpening structural details without introducing artifacts.

\subsection{Ablation Studies}

To validate the contribution of each key component, we conducted a series of ablation studies by systematically removing individual modules. As shown in Table~\ref{tab:ablation}, the complete model consistently outperforms all ablated variants across all No-Reference Image Quality Assessment (NR-IQA) metrics. Notably, removing the PDE optimization stage degrades the NRBP score from 76.82 to 41.49, highlighting the critical role of iterative refinement. Similarly, disabling core components like nonlocal regularization or the edge-preserving term substantially impairs image naturalness and structure, evidenced by a collapse in NIQE and BRISQUE scores. These quantitative results are visually corroborated by Fig.~\ref{fig:placeholder2}, where ablated models exhibit residual haze and loss of detail, while the full model yields the clearest outcomes. This analysis confirms that each component is integral to achieving high-fidelity image dehazing.

\vspace{-2mm}
\section{Conclusion}
This paper presents a novel PDE-based framework for single-image dehazing that integrates the atmospheric scattering model, nonlocal regularization, and dark channel prior. The key innovations include:
1. A mathematically well-posed PDE model that ensures existence and uniqueness of weak solutions.
2. An adaptive regularization strategy that adjusts to local haze concentration.
3. An efficient GPU-accelerated fixed-point iteration scheme for real-time processing.

% Experimental results on the RESIDE dataset demonstrate that the method outperforms state-of-the-art approaches in both quantitative metrics (PSNR 24.17 dB, SSIM 0.91) and visual quality. The framework bridges traditional physical modeling and data-driven approaches, offering a principled alternative with clear physical interpretation.
% Experimental results on a diverse real-world dataset demonstrate that the proposed method outperforms state-of-the-art approaches across a comprehensive suite of no-reference image quality assessment metrics. Notably, our framework achieves leading performance in key metrics, securing top scores in NIQE (4.51), BRISQUE (28.46), PIQE (35.68), DHQI (57.85), and NRBP (76.82). The framework bridges traditional physical modeling and data-driven approaches, offering a principled alternative with clear physical interpretation and superior visual quality. 
Experimental results on a real-world dataset confirm our method's superiority over state-of-the-art competitors, achieving leading scores across five no-reference IQA metrics. This framework bridges physical modeling with mathematical principles, offering a physically interpretable and high-fidelity alternative to purely data-driven techniques.

% {
% \small
% \bibliographystyle{IEEEtran}
% \bibliography{bibliography}
% }
{

}
\newpage
\appendix

\section{Numerical Implementation}
\subsection{Fixed-Point Iteration Scheme}
The PDE is solved using an adaptive fixed-point iteration method. At each iteration \(n\), the solution is updated as:
\[
\begin{aligned}
u^{(n+1)} ={}& u^{(n)} + \tau \Big[ \text{div}(D_n\nabla u^{(n)}) \\
             & - \lambda(t) G(u^{(n)}) + \Phi(I,t,A) \Big]
\end{aligned}
\]
where \(D_n = D(\nabla u^{(n)})\) and \(\tau\) is the relaxation parameter. The scheme is derived by linearizing the nonlinear diffusion term and using an explicit update. The relaxation parameter \(\tau\) must satisfy:
\[
0 < \tau < \frac{2}{\|D(\nabla u)\Delta\| + \lambda(t)M}
\]
to ensure convergence, where \(M\) is the bound of \(G\). In practice, \(\tau = 0.2\) is found to be stable across different haze conditions.

\subsection{Discretization of PDE Operators}
\subsubsection{Diffusion Term Discretization}
The divergence term \(\text{div}(D(\nabla u)\nabla u)\) is discretized using central differences on a uniform grid. For a 2D image with pixel \((i,j)\), the x-derivative of \(D(\nabla u)u_x\) is:
\[
\begin{aligned}
&\left[\frac{\partial}  {\partial x}\left(D(\nabla u)  u_x  \right)\right]_{i,j}  \approx{}  \\
& \frac{D_{i+1/2,j} u_x^{i+1/2,j} - D_{i-1/2,j} u_x^{i-1/2,j}}{\Delta x}
\end{aligned}
\]
where \(D_{i+1/2,j} = D\left(\frac{u_{i+1,j} - u_{i,j}}{\Delta x}, \frac{u_{i,j+1} - u_{i,j-1}}{2\Delta y}\right)\) and \(u_x^{i+1/2,j} = \frac{u_{i+1,j} - u_{i,j}}{\Delta x}\). Similarly for the y-derivative, with \(\Delta x = \Delta y = 1\) for pixel grids.

\subsubsection{Nonlocal Term Implementation}
The Gaussian convolution \(G(u)\) is efficiently computed using \(\text{scipy.ndimage.gaussian\_filter}\) in Python, or via PyTorch's built-in convolution for GPU acceleration. The kernel size is set to \(5 \times 5\) with \(\sigma = 2.0\), balancing computational efficiency and nonlocal interaction range.

\subsection{User study}
To evaluate the subjective visual perception of our proposed method against other methods, we conducted a user study. We invited five experts with an image processing background and 16 naive observers as testers. These testers were instructed to focus on three primary aspects: (i) Haze density compared to the original hazy image, (ii) Clarity of details in the dehazed image, and (iii) Color and aesthetic quality of the de-hazed image. The results for each method, along with the corresponding hazy images, were presented to the testers anonymously. They were asked to score each method on a scale from 1 (worst) to 10 (best). The haze images were randomly selected, with a total of 50 images from our real-world dataset. The user study scores are reported in Table~\ref{tab:user_study}, showing that our method achieved the highest average score.

\begin{table}[h!]
\centering
\caption{User study scores on the real-world dataset.}
\label{tab:user_study}
\resizebox{\linewidth}{!}{%
\begin{tabular}{l|cccccccc}
\hline
\textbf{Dataset} & \textbf{DCP} & \textbf{CAP} & \textbf{Color-Lines}& \textbf{BCCR} & \textbf{VarFusion} & \textbf{CORUN} & \textbf{DiffDehaze } & \textbf{Ours} \\ \hline
Score            & 5.31             & 5.19     & 5.65         & 5.82               & 6.55                   & 6.81                & 7.03                     & \textbf{8.79} \\ \hline
\end{tabular}
}
\end{table}

\begin{figure}
    \centering
    \includegraphics[width=0.95\linewidth]{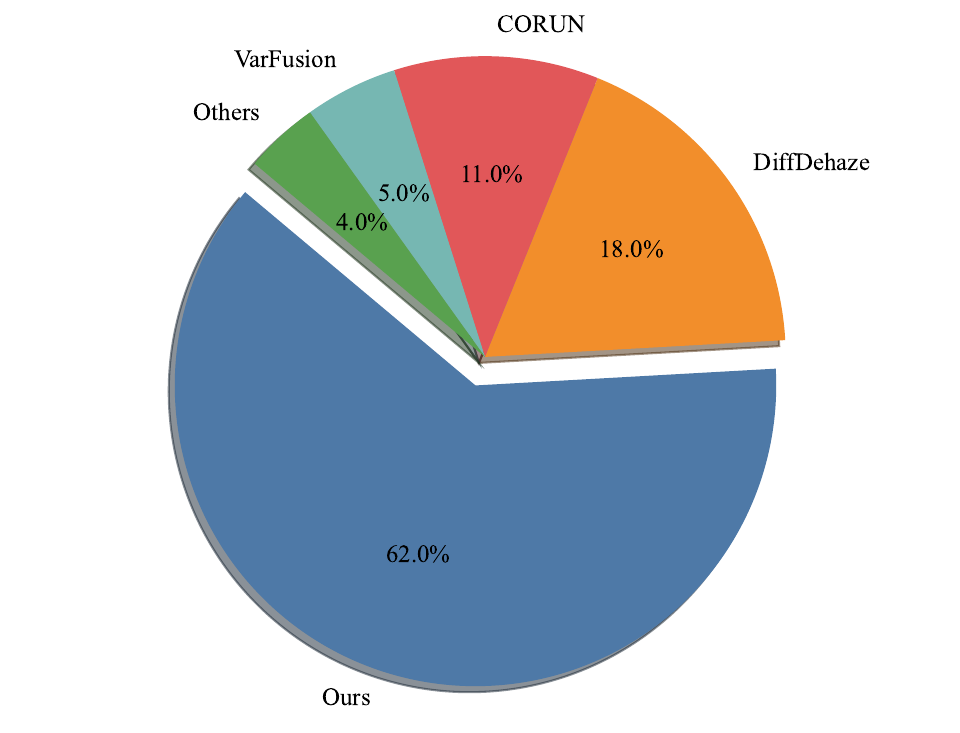}
    \caption{Distribution of best visual quality votes.}
    \label{fig:placeholder2}
    \vspace{-4mm}
\end{figure}

To evaluate subjective visual quality, we conducted a user preference study where participants, including five imaging experts and 16 naive observers, rated images based on haze removal, clarity, and color naturalness. The study used a total of 50 randomly selected real-world images.

As illustrated in the statistical summary in Fig.~\ref{fig:placeholder2}, our proposed method was judged to have the "Best Visual Quality" in a clear majority of cases (62\%). This indicates our approach excels at removing haze without introducing common artifacts, while effectively maintaining fine-grained textures and high color fidelity. Although deep learning methods like DiffDehaze and CORUN performed respectably, receiving 18\% and 11\% of the votes respectively, they occasionally failed to preserve details. Traditional and fusion-based approaches, such as VarFusion (5\%), proved least effective. In summary, the study validates that our method's quantitative superiority translates to results that are also perceptually preferred by human observers.

\end{document}